\documentclass{article}

\usepackage[preprint]{neurips_2025}

\usepackage[utf8]{inputenc} % allow utf-8 input
\usepackage[T1]{fontenc}    % use 8-bit T1 fonts
\usepackage{hyperref}       % hyperlinks
\usepackage{url}            % simple URL typesetting
\usepackage{booktabs}       % professional-quality tables
\usepackage{amsfonts}       % blackboard math symbols
\usepackage{nicefrac}       % compact symbols for 1/2, etc.
\usepackage{microtype}      % microtypography
\usepackage{xcolor}         % colors
\usepackage{graphicx}
\usepackage{amsmath}
\usepackage{subcaption}
\usepackage{amsthm}

\newtheorem{definition}{Definition}
\newtheorem{proposition}[definition]{Proposition}
\title{Multi-granularity Adaptive Hypergraph Representation Learning via Granular-ball}

\author{%
  Sen Zhao\\
  Chongqing University of \\
  Posts and Telecommunications\\
  \And
  Yifan Guan*\\
  Chongqing University of \\
  Posts and Telecommunications\\
  \And
  Jinyuan Ni\\
 Chongqing University of \\
  Posts and Telecommunications\\
  \And
  Gaojie Xu\\
  Chongqing University of \\
  Posts and Telecommunications\\
  \AND
  Zhang Xu\\
  \And
  Xiaoyu Lian\\
  \And
  Yi Liu\\
  \And
  Yi Wang\\
  \And
  Wei Wang
}

\begin{document}

\maketitle

\begin{abstract}
  Hypergraph representation learning aims to capture high-order information in graphs by constructing hyperedges that simultaneously connect multiple nodes. These hyperedges adapt to the graph's topological features, facilitating the extraction of high-order relationships at multiple granularities. Most prior work relies on predefined definitions to generate hyperedges, overlooking the diversity in graph topological structures and the multi-granularity characteristics of hyperedges. As a result, this limits their ability to effectively and adaptively discover high-order relationships and efficiently process complex structural information. To address this limitation, we propose a novel framework called \underline{M}ulti-\underline{G}ranularity \underline{H}ypergraph \underline{R}epresentation \underline{L}earning (MGHRL). MGHRL introduces an Adaptive Granular Hypergraph Generation strategy, which generates hyperedges at multiple levels of granularity through the adaptive splitting of granular-ball, effectively capturing high-order relationships based on the graph's topological structure. Additionally, we propose a Multi-Granularity Hypergraph Network with multiple sub-networks, capturing features from hyperedges at different granularities and integrating them via hierarchical reversible connections. Experimental results show that MGHRL significantly outperforms baseline models on benchmark datasets.%, particularly in terms of test accuracy on social network datasets. %The code is available at https://anonymous.4open.science/r/MGHRL.
\end{abstract}

\section{Introduction}
\label{Introduction}

Traditional graphs capture only pairwise relationships, which restricts their capacity to model high-order interactions common in real-world systems like social networks or bioinformatics~\cite{kipf2017semi,hamilton2017inductive,wu2020comprehensive}. To overcome this constraint, researchers have proposed hypergraph structures, which use hyperedges to simultaneously connect multiple nodes, enabling better modeling of complex relations.~\cite{NIPS2006_dff8e9c2,battaglia2018relational}. 

Early work~\cite{ramadan2004hypergraph,zhou2006learning} modeled such structures with hypergraphs, but capturing these interactions poses technical challenges~\cite{madan2011sensing,li2010computational} and, even when observable, the hypergraph version of many datasets~\cite{newman2004coauthorship,sarigol2014predicting} are often unpublished. To address this issue, heuristic approaches~\cite{feng2019hypergraph,DBLP:journals/pami/GaoFJJ23} are used to generate hyperedges, such as k-Hop neighbors hypergraphs and k-Nearest neighbors hypergraphs. Hypergraph Neural Networks (HGNNs)~\cite{feng2019hypergraph,DBLP:journals/pami/GaoFJJ23,DBLP:conf/iclr/WangYLWL23} use these hyperedges for message passing between nodes and hyperedges. Although effective, previous approaches often rely on predefined rules and overlook the diversity of graph topological structures and the multi-granularity nature of hyperedges.

\begin{figure}[htbp]
  \centering
  \includegraphics[width=0.95\linewidth]{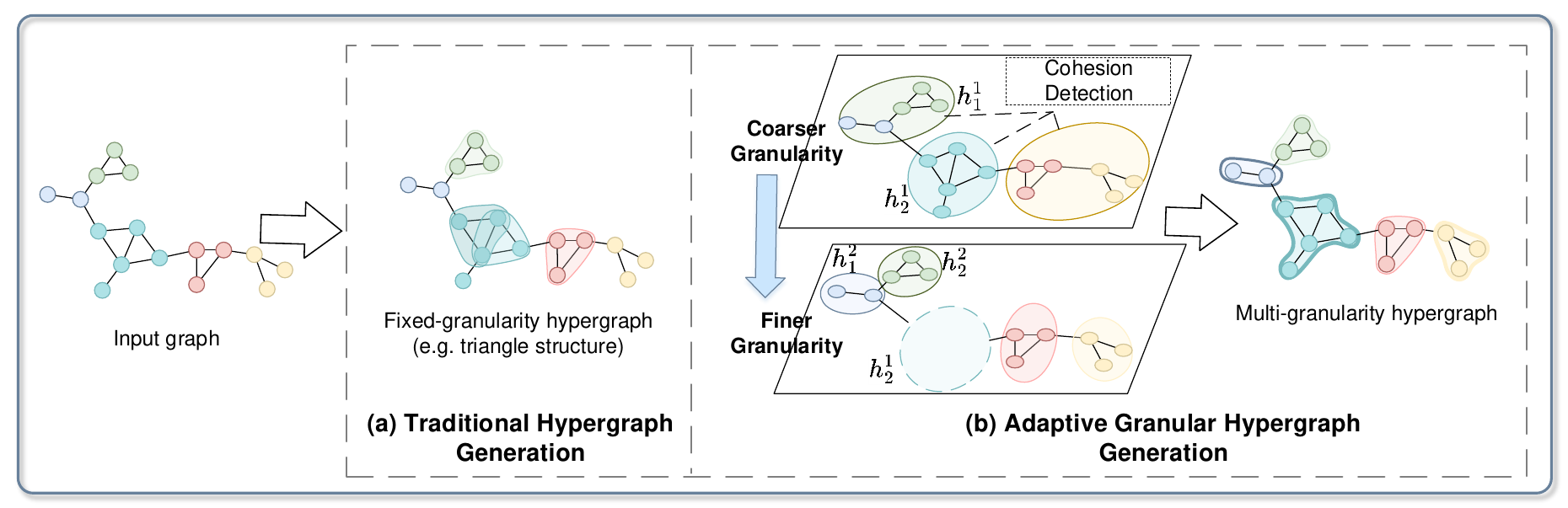}
  \caption{Comparison of traditional fixed-granularity and proposed adaptive multi-granularity hypergraph generation.(a) Traditional methods (e.g., triangle-based) impose fixed hyperedges, overlooking hierarchical structures.(b) Our method adaptively splits hyperedges by detecting node cohesion, refining from coarse to fine granularity to capture richer topologies.}
  \label{fig:motivation}
\end{figure}

In real-world hypergraph scenarios, hyperedges are formed by relationships among cohesive groups of nodes, which span multiple levels of granularity. Figure~\ref{fig:motivation} compares traditional fixed-granularity and our adaptive approach. Traditional methods (e.g., triangle-based) ignore hierarchical structure. In contrast, our method dynamically splits hyperedges by detecting node cohesion, refining from coarse to fine granularity to capture richer topologies. For example, $h_2^1$ reflects optimal granularity, while $h_1^1$ can be split into $h_1^2$ and $h_2^2$ for finer structural insights.

Modeling the multi-granularity nature of hyperedges is a non-trivial task, primarily due to two key challenges: (1) Hyperedge generation requires adaptivity, as cohesive node groups often have complex topologies not captured by fixed rules. Drawing inspiration from the multi-granularity distinction in Granular-Ball Computing~\cite{xia2019granular,xie2023efficient}, we propose generating hyperedges at multiple levels of granularity through the adaptive splitting of granular-ball, effectively capturing high-order relationships based on the graph’s topological structure. (2) Hypergraph modeling is often granularity-agnostic. Most existing methods primarily focus on information propagation between nodes and hyperedges, neglecting the distinct learning capabilities associated with hyperedges of different granularities. This oversight significantly constrains the model's capacity to effectively capture and represent multi-level relational patterns.

To address the aforementioned issues, we propose a novel  \underline{M}ulti-\underline{G}ranularity \underline{H}ypergraph \underline{R}epresentation \underline{L}earning (MGHRL) framework. MGHRL introduces an Adaptive Granular Hypergraph Generation strategy, which progressively constructs multi-granularity hypergraphs through adaptive splitting. This approach generates the hypergraph in a coarse-to-fine manner, ensuring that the generated hyperedges accurately reflect the relationships between nodes and the overall graph topology~\cite{liu2023multi,kang2022dynamic}. Additionally, we propose a novel Multi-Granularity Hypergraph Network, which consists of multiple sub-networks, referred to as columns. Each column learns independently from distinct graph replicas, with decoupled parameters. This enables each column to extract feature representations from hyperedges of varying granularities independently, while progressively accumulating and refining information through multi-level reversible connections between columns.

In a nutshell, this work makes the following contributions:
\begin{itemize}
    \item We stress the critical role of the multi-granularity nature of hyperedges in hypergraph representation learning.
    
    \item  We propose a novel MGHRL framework that generates hyperedges at multiple granularities through the adaptive splitting of granular balls. Additionally, we introduce the Multi-Granularity Hypergraph Network, which independently extracts feature representations from hyperedges of varying granularities and progressively refines information through multi-level reversible connections between these hyperedges.
    
    \item  We conducted extensive experiments on eight real-world datasets and MGHRL effectively improves the performance of hypergraph representation learning.
\end{itemize}

\section{Related Work}
Graph Neural Networks (GNNs) have greatly advanced the representation and analysis of graph-structured data~\cite{kipf2017semi,velivckovic2018graph,hamilton2017inductive}. Traditional Graph Convolutional Networks (GCNs) aggregate features from neighboring nodes through localized convolution, excelling in semi-supervised tasks~\cite{kipf2017semi}. Building on GCNs, Graph Attention Networks (GATs) introduce attention mechanisms to prioritize important neighbors during aggregation~\cite{velivckovic2018graph}, while GraphSAGE~\cite{hamilton2017inductive} enables inductive learning by sampling fixed-size neighborhoods for large-scale graphs.
Graph Isomorphism Networks (GINs)~\cite{DBLP:conf/iclr/XuHLJ19} use a robust aggregation function to distinguish between diverse graph structures. Despite effectiveness, traditional graphs are limited to capturing only pairwise relationships, which constrains their ability to model high-order interactions~\cite{kipf2017semi,hamilton2017inductive}.

To overcome the limitations of traditional graph structures, hypergraphs~\cite{ramadan2004hypergraph,zhou2006learning} have emerged as a more expressive graph structure for representing high-order relationships, where hyperedges connect multiple nodes simultaneously. Recent advancements in hypergraph neural networks (HNNs)~\cite{feng2019hypergraph,kang2022dynamic} have enhanced the representation of complex high-order relationships within hypergraphs, addressing the limitations of traditional GNNs in capturing connections that extend beyond simple pairwise relationships~\cite{ding2020more}. HyperGCN~\cite{yadati2019hypergcn} utilizes a clique expansion technique to convert real-world hypergraphs into simple graphs, enabling the use of standard graph convolution operations. However, the acquisition of hyperedge data faces technical challenges in many domains~\cite{madan2011sensing,li2010computational}, and such data is often not publicly available. 
HGNN~\cite{feng2019hypergraph} constructs hyperedges by connecting the k nearest nodes to a central node based on distance or similarity, and incorporates spectral convolutions tailored for hypergraph structures. It aggregates node features to generate hyperedge representations, which are then used to update adjacent node features, allowing the model to learn intricate relational patterns within the hypergraph.
UniGCN and UniGAT~\cite{ijcai21-UniGNN} represent a significant advancement by integrating graph and hypergraph neural networks into a unified framework. This integration enables a more comprehensive understanding of information propagation across both network structures, improving the model’s ability to transfer and analyze information across diverse data representations. Additionally, HGNN+~\cite{DBLP:journals/pami/GaoFJJ23} build hyperedges with k-Hop neighbors and introduce a versatile framework for high-order and multi-modal analysis, which captures latent correlations within each modality and effectively integrates them across modalities.
 ED-HNN~\cite{DBLP:conf/iclr/WangYLWL23} leverages star expansions and message passing neural networks to efficiently approximate continuous equivariant hypergraph diffusion operators, enabling the model to capture complex high-order relationships and improve performance on heterophilic hypergraphs.
 HyperNAS~\cite{DBLP:conf/aaai/LinP0J24} further introduces a hypergraph neural architecture search method that defines a specialized search space, uses a differentiable search algorithm, and optimizes the architecture with a hypergraph structure-aware criterion.  However, previous methods often rely on predefined criteria for generating hyperedges, ignoring the diversity of graph structures and the multi-granularity nature of hyperedges.

\section{Approach}
\label{Approach}

Figure \ref{fig: framework} presents the proposed framework, comprising the Adaptive Granular Hypergraph Generation (AGHG) strategy and the Multi-Granularity Hypergraph Network (MGHN). AGHG initializes coarse-grained granular balls and adaptively refines them via binary splitting based on node features and topology. MGHN employs a multi-column architecture with reversible connections to integrate information across granularities, enhancing the modeling of complex hypergraph structures.

\subsection{Adaptive Granular Hypergraph Generation}
\begin{figure*}[t]
  \centering
  \includegraphics[width=0.95\linewidth]{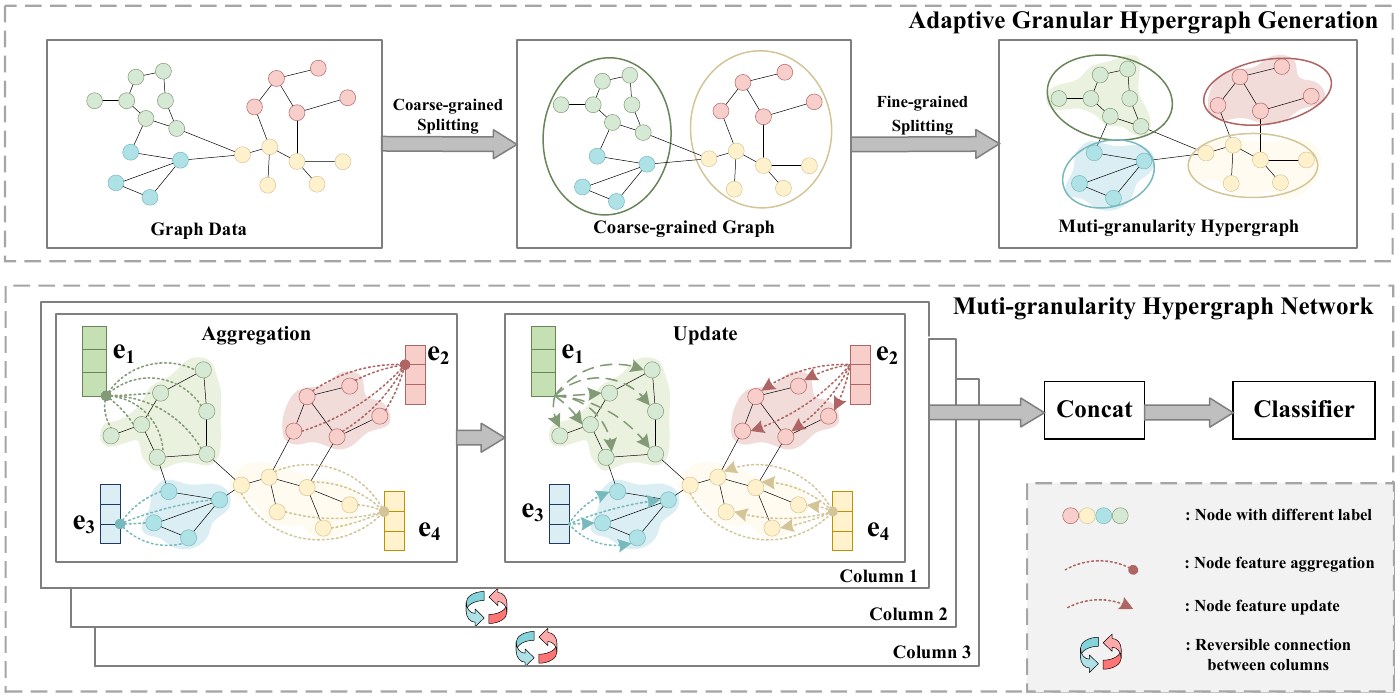}
  \caption{Overview of our proposed MGHRL framework. It mainly contains two modules: Adaptive granular hypergraph generation process, and Multi-granularity hypergraph network architecture.}
  \label{fig: framework}
\end{figure*}
This section provides a detailed explanation of AGHG. To formalize the approach, we introduce the Granular-ball method, which refines a graph into multiple subdomains.

\begin{definition}
\label{def.granular}
\textbf{Topological Granular-ball Decomposition.} Let $\mathcal{G} = (\mathcal{V}, \mathcal{E}, \mathbf{X}, \mathbf{A})$ denote a graph-based structured domain endowed with a metric $d : \mathcal{V} \times \mathcal{V} \rightarrow \mathbb{R}_{\geq 0}$, where $\mathbf{X}$ represents node features and $\mathbf{A}$ the adjacency matrix. This metric induces a topology $\tau_d$ on $\mathcal{V}$, in which basic neighborhoods are defined by closed metric balls.

A Granular-ball $\widetilde{\mathcal{G}}_i$ is defined as a closed ball in $(\mathcal{V}, d)$:
\begin{equation}
\widetilde{\mathcal{G}}_i=\left\{v \in \mathcal{V} \mid d\left(v, c_i\right) \leq r_{c_i}\right\},
\end{equation}
where $c_i \in \mathcal{V}$ is the center node and $r_{c_i}$ is a radius selected to preserve local connectivity and semantic cohesion.
The Granular-ball Decomposition of $\mathcal{G}$ is a finite collection of such subspaces:
\begin{equation}
    \widetilde{\mathbb{G}}=\left\{\left(\widetilde{\mathcal{V}}_i, \tilde{\mathcal{E}}_i\right)\right\}_{i=1}^n,
\end{equation}
where the family $\{ \widetilde{\mathcal{V}}_i \}_{i=1}^n$ form a topological covering of $\mathcal{V}$.

We aim to find a covering $\widetilde{\mathbb{G}}$ such that the subspace structure is both topologically coherent and semantically meaningful, formalized as the following optimization:
\begin{equation}
\min _{\widetilde{\mathbb{G}}} \sum_{i=1}^n \frac{1}{Q\left(\widetilde{\mathcal{G}}_i\right)}+\lambda \cdot n,
\end{equation}
where $Q(\widetilde{\mathcal{G}}_i)$ denotes the topological or structural quality of subgraph $\widetilde{\mathcal{G}}_i$.
%Given a graph \( \mathcal{G} = (\mathcal{V}, \mathcal{E}) \), where \(\mathcal{V}\) denotes the set of \(n\) nodes and \(\mathcal{E}\) denotes the set of \(m\) edges, the Granular-ball method refines \( \mathcal{G} \) into multiple subdomains as follows:
%\begin{equation}
%\mathcal{G} = \{\widetilde{\mathcal{G}}_1, \widetilde{\mathcal{G}}_2, \dots, \widetilde{\mathcal{G}}_n\}, \quad \text{s.t.} \quad \min \sum_{j=1}^{n} \left( \frac{1}{\text{quality}(\widetilde{\mathcal{G}}_j)} \right) + \lambda \cdot n
%\end{equation}
%where each granular ball \( \widetilde{\mathcal{G}}_i \) is defined as:
%\begin{equation}
%\widetilde{\mathcal{G}}_i = (\widetilde{\mathcal{V}}_i, \widetilde{\mathcal{E}}_i), \quad \text{where} \quad \widetilde{\mathcal{V}}_i \subseteq \mathcal{V} \quad \text{and} \quad \widetilde{\mathcal{E}}_i \subseteq \mathcal{E}.
%\end{equation}
\end{definition}

As shown in Figure~\ref{fig: framework}, the initial graph \( \mathcal{G} \) is regarded as a coarse-grained granular ball and iteratively decomposed through adaptive binary splitting. At each step $t$, a granular ball \( \widetilde{\mathcal{G}}^{t} \)  is partitioned into finer sub-balls  \( \{\widetilde{\mathcal{G}}^{t+1}\} \), guided by structural and feature similarity. This process yields a topological covering of \( \mathcal{G} \) that captures hierarchical correlations—from fine-grained local patterns to global structures—enabling the construction of expressive hyperedges.
%As illustrated in Figure \ref{fig: framework}, "Adaptive Granular Hypergraph Generation," the initial graph \( \mathcal{G} \) is initially treated as a single coarse-grained granular ball. It is then divided into coarse-grained granular balls, followed by further iterative binary splitting into finer granular balls based on adaptive conditions. Each granular ball comprises closely connected nodes and edges with similar features, effectively capturing high-order correlations and generating appropriate hyperedges. This iterative process uncovers structural relationships at various hierarchical levels, ranging from fine-grained local interactions to coarse-grained global structures. At step \( t \) of the splitting phase, the coarse-grained granular ball \( \widetilde{\mathcal{G}}^{t} \) is divided into finer granular balls \( \{\widetilde{\mathcal{G}}_1^{t+1}, \widetilde{\mathcal{G}}_2^{t+1}, \dots, \widetilde{\mathcal{G}}_n^{t+1}\} \).

\subsubsection{Coarse-grained granular-ball initialization}
The 0-th step in the splitting phase is initialization, where a fixed number of granular balls are generated for coarse data partitioning. Following previous work~\cite{yu2001upper}, the initial number is set to \( q = \sqrt{N} \). First, the \( q \) nodes with the highest degrees are selected as the initial centers of the coarse-grained granular balls, denoted by \( \mathcal{C} = \{c_0, c_1, \dots, c_q\} \), where \( \mathcal{C} \) represents the set of central nodes. Then, a breadth-first search (BFS) algorithm assigns the remaining nodes in the graph \( \mathcal{G} \) to the closest center node, as determined by the following formula:
\begin{equation}\label{Assign}
\text{Assign}(v) = \arg \min_{c_i \in C} \text{d}(v, c_i),
\end{equation}
where \( \text{d}(v, c_i) \) is the shortest path distance from node \( v \) to center node \( c_i \).

After assignment, the initial granular space $\widetilde{\mathbb{G}}_{init}$ is constructed, where each coarse-grained granular ball \( \widetilde{\mathcal{G}}_i \) is defined as:
\begin{equation}\label{Gi}
\widetilde{\mathcal{G}}_i = \{ v \in V \mid \text{Assign}(v) = c_i \} \cup \{ c_i \}.
\end{equation}

%This results in the initial granular-ball distribution:

%\begin{equation}
%\widetilde{\mathcal{G}}_{\text{init}} = \{\widetilde{\mathcal{G}}_1^0, \widetilde{\mathcal{G}}_2^0, \dots, \widetilde{\mathcal{G}}_n^0\}.
%\end{equation}

This method ensures a balance between granular-ball quality and algorithm efficiency, avoiding excessive computational complexity and ensuring scalability for large graphs. After initialization, each \( \widetilde{\mathcal{G}}_i^0 \) proceeds to the iterative binary splitting phase.

\subsubsection{Fine-grained binary splitting}
At step \( t \), a granular ball \( \widetilde{\mathcal{G}}_i^t \) is split into two finer granular balls, \( \widetilde{\mathcal{G}}_1^{t+1} \) and \( \widetilde{\mathcal{G}}_2^{t+1} \), for the next iteration. The splitting process begins by identifying two distinct nodes \( u, v \in \widetilde{\mathcal{G}}_i^t \) with the highest degrees to serve as the centers of the child granular balls. Specifically, all candidate node pairs $(x,y)$ within \( \widetilde{\mathcal{G}}_i^t \) such that $x\neq y$ are evaluated, and the pair with the highest degrees is selected according to:
%The splitting process begins by selecting the two nodes \( u, v \in \widetilde{\mathcal{G}}_i^t \) with the highest degrees to serve as centers of the child granular balls. Specifically, candidate nodes \( x \) and \( y \) within \( \widetilde{\mathcal{G}}_i^t \) are evaluated, and the nodes with the highest degrees are selected as:
\begin{equation}
u, v = \arg \max_{x, y \in \widetilde{\mathcal{G}}_i^t} (\text{deg}(x), \text{deg}(y)), \quad x \neq y.
\end{equation}
After selecting the center nodes, the remaining nodes in \( \widetilde{\mathcal{G}}_i^t \) are assigned to these centers, forming the child granular balls \( \widetilde{\mathcal{G}}_1^{t+1} \) and \( \widetilde{\mathcal{G}}_2^{t+1} \), based on equations (\ref{Assign}) and (\ref{Gi}).

To effectively evaluate the quality of granular balls, we focus on their connectivity, which reflects how tightly nodes are linked within each subgraph. A higher connectivity indicates stronger cohesion and richer high-order relationships, which are crucial for accurate hypergraph representation. Therefore, we define the connectivity evaluation as:
\begin{equation}
Q(\mathcal{G}) = \frac{2|\mathcal{E}|}{|\mathcal{V}|}.
\end{equation}
Finally, the decision to continue the splitting process is made based on the following adaptive criteria: 
\begin{equation}
Q(\widetilde{\mathcal{G}}_1^{t+1}) + Q(\widetilde{\mathcal{G}}_2^{t+1}) > Q(\widetilde{\mathcal{G}}_i^t),
\end{equation}
where the splitting process terminates successfully when the combined quality of the child granular balls surpasses that of the parent; if not, the process proceeds recursively, thereby maintaining both computational efficiency and accuracy.

After iterative splitting, each granular ball in \( \widetilde{\mathbb{G}}_{\text{init}} \) is adaptively and iteratively split into multiple finer-grained granular balls, resulting in \( \widetilde{\mathbb{G}}_{\text{final}} \). These granular balls capture the multi-granular, high-order relationships among nodes. From this, we generate multi-granularity hyperedges as shown below:
\begin{equation}
\hat{\mathcal{E}} = \{ \mathcal{V}_i \mid \mathcal{G}_i  \in \widetilde{\mathbb{G}}_{\text{final}} \},
\end{equation}
The hyperedge is defined by the granular balls generated during the process. Ultimately, a hypergraph \( \hat{\mathcal{G}} = (\hat{\mathcal{V}}, \hat{\mathcal{E}}) \) is obtained, containing multi-granularity hyperedges, with \( \hat{\mathcal{V}} = \mathcal{V} \).

\begin{proposition}
\label{prop.metricl}
(Proof in Appendix \ref{app.proof.metric}.) 
The distance function $d(v,c)$ defines a valid metric over the structured domain $(\mathcal{V},d)$, and the topology it induces supports the Granular-ball covering introduced in Definition \ref{def.granular}.
\end{proposition}

\begin{proposition}
\label{prop.granularball}
(Proof in Appendix \ref{app.proof.granularball}.) 
The Granular-ball covering $\widetilde{\mathbb{G}}$ defined in Definition \ref{def.granular} forms a finite topological covering of $\mathcal{V}$ under the topology $\tau_d$, satisfying full coverage, local coherence, and neighborhood basis property.
\end{proposition}

\subsection{Multi-Granularity Hypergraph Network}
As illustrated in Figure \ref{fig: framework}, "Multi-Granularity Hypergraph Network,"the network architecture is composed of multiple sub-networks, or columns. Each column can be implemented using a traditional single-column architecture and consists of multiple blocks. The columns communicate and aggregate information via multi-level reversible connections. This design enables progressive feature decoupling during forward propagation, preserving total information content, unlike traditional networks that may compress or discard features. To retain global topology, a full copy of the hypergraph is fed into each column. The reversible connections between columns are defined as follows: assuming each column consists of \( m \) blocks, the forward and backward propagation for the \( t \)-th block are:
\begin{align}
  & \text{Forward: }  x_t = F_t(x_{t-1}, x_{t-m+1}) + \gamma x_{t-m}\label{eq:forward},\\
  & \text{Backward: }  x_{t-m} = \gamma^{-1} [x_t - F_t(x_{t-1}, x_{t-m+1})],
\end{align}
where \( F_t \) represents a non-linear operation similar to residual functions in ResNets, and \( \gamma \) is a reversible operation, such as channel-wise scaling, with \( \gamma^{-1} \) as its inverse. 

Compared to conventional architectures, MGHN provides three key advantages:  
\begin{itemize}
    \item \textbf{Feature Disentangling:} In MGHN, low-level features are typically retained in the lower columns, while high-level semantic features are gradually captured and extracted in the higher columns. These features are effectively fused through reversible connections.
    \item \textbf{Multi-Granularity Learning:} Each column in the MGHN learns from a distinct graph replica with decoupled parameters, allowing it to focus on different granularities of hyperedges. Reversible connections between columns help refine and accumulate information, integrating both coarse and fine details.
    \item \textbf{New Scaling Factor for Large Models:} Unlike traditional architectures that scale in depth and width, the multi-granularity hypergraph network introduces columns as an additional scaling dimension for GNNs. Adding more columns enhances the model's ability to capture both low-level and high-level features.
\end{itemize}

\paragraph{Hypergraph Attention Convolution Block.}  
Every column of the MGHN comprises four hypergraph attention convolution blocks. During the feature extraction process, the following two assumptions are made~\cite{BAI2021107637}:

\begin{itemize}
    \item \textbf{Stronger Propagation Among Nodes Connected by the Same Hyperedge:} Nodes linked by the same hyperedge effectively share information due to their similar features.
    \item \textbf{Greater Trust in Heavier Hyperedges During Propagation:} Heavier hyperedges represent stronger connections, making them more reliable for effective information exchange.
\end{itemize}

The update step for hypergraph convolution is defined as follows, based on the outlined assumptions:

Given a hypergraph \( \hat{\mathcal{G}} = (\hat{\mathcal{V}}, \hat{\mathcal{E}}) \), where \( \hat{\mathcal{V}} \) denotes the set of nodes and \( \hat{\mathcal{E}} \) denotes the set of hyperedges:

For each node \( \hat{v}_i \) and its associated hyperedges \( \hat{e}_j \in \hat{\mathcal{E}} \) (i.e., the set of hyperedges connected to node \( \hat{v}_i \)), the attention score is calculated as:
\begin{equation}
\alpha_{\hat{v}_i\hat{e}_j} = \sigma \left( a^T [W h_{\hat{v}_i} ; W h_{\hat{e}_j}] \right),
\end{equation}
where \(\sigma\) denotes a nonlinear activation function. Here, \(a \in \mathbb{R}^{2d'}\) is a learnable attention parameter vector, and \(W\) is a learnable weight matrix used to transform node and hyperedge features. The vectors \(h_{\hat{v}_i}\) and \(h_{\hat{e}_j}\) represent the embeddings of node \(\hat{v}_i\) and hyperedge \(\hat{e}_j\), respectively, with the latter aggregated from its connected nodes.
% where \( \alpha_{\hat{v}_i\hat{e}_j} \) represents the attention score between node \( \hat{v}_i \) and hyperedge \( \hat{e}_j \), \( \sigma \) is a nonlinear activation function, \( a \in \mathbb{R}^{2d'} \) is a learnable attention parameter vector, \( W \) is a learnable weight matrix for transforming node and hyperedge features, \( h_{\hat{v}_i} \) is the embedding vector of node \( \hat{v}_i \), and \( h_{\hat{e}_j} \) is the embedding vector of hyperedge \( \hat{e}_j \), aggregated from its connected nodes connected to hyperedge \( \hat{e}_j \).

The attention scores for all hyperedges adjacent to node \( \hat{v}_i \) are normalized via softmax:
\begin{equation}
\tilde{\alpha}_{\hat{v}_i\hat{e}_j} = \frac{\exp(\alpha_{\hat{v}_i\hat{e}_j})}{\sum_{\hat{e}_j' \in \hat{\mathcal{E}}} \exp(\alpha_{\hat{v}_i\hat{e}_j'})},
\end{equation}
where \( \tilde{\alpha}_{\hat{v}_i\hat{e}_j} \) is the normalized attention score between node \( \hat{v}_i \) and hyperedge \( \hat{e}_j \), and \( \exp(\alpha_{\hat{v}_i\hat{e}_j}) \) denotes the exponential of the raw attention score.

Node  \( \hat{v}_i \)'s embedding is updated by aggregating a weighted sum of hyperedge embeddings:
\begin{equation}
\tilde{x}_i = \sum_{\hat{e}_j' \in \hat{\mathcal{E}}} \tilde{\alpha}_{\hat{v}_i\hat{e}_j} W h_{\hat{e}_j}.
\end{equation}

\paragraph{Fusion Module.} The non-linear residual module is a key fusion component in the MGHN. It integrates feature representations from both the current and previous columns using residual connections, thereby facilitating the fusion of multi-granularity features. As in the case of the \( t \)-th block in formula~\ref{eq:forward}, this fusion module is expressed as follows:
\begin{align}
    F(x_{t-1}, x_{t-m+1}) = \sigma(\text{LayerNorm}(W_1 x_{t-1} + b_1)) + \beta x_{t-m+1},
\end{align}
where \( x_{t-m+1} \) represents the feature vector from the previous column, enabling cross-layer information integration. \( \beta \) is a learnable scaling parameter that adaptively adjusts the contributions of features from both the current and previous columns. Controls the intensity of feature fusion, thereby enhancing the propagation of information across layers.

\paragraph{Reversible Operation.} The reversible scaling module improves training stability by applying channel-level scaling. It employs a learnable parameter, initially set to 1, which adapts during training. The module combines the current features with scaled previous features to enable efficient information transfer while suppressing excessive magnitudes, thus maintaining stability. To prevent numerical errors during backpropagation, the scaling parameter is constrained to a minimum absolute value of \(1 \times 10^{-3}\), thereby ensuring stable gradient flow and effective training of complex architectures.

\section{Experiment}
\label{Experiment}
\subsection{Datasets}
\label{Datasets}
This experiment utilizes seven public benchmark datasets: three citation networks (Cora~\cite{Sen_Namata_Bilgic_Getoor_Galligher_Eliassi-Rad_2017}, Citeseer~\cite{Getoor_2012}, PubMed~\cite{Getoor_2012}) and four social media networks (Flickr~\cite{10.1007/978-3-642-33765-9_59}, BlogCatalog~\cite{10.5555/2888116.2888372}, Facebook~\cite{Rozemberczki_Allen_Sarkar_2019}, GitHub~\cite{Rozemberczki_Allen_Sarkar_2019}). In citation networks, nodes represent academic papers with sparse bag-of-words features, and edges denote citation relationships, with the objective of predicting paper topics. In social media networks, nodes represent developers or websites, with features derived from user information or page themes, and the task is to predict the fields of developers or the categories of pages. The statistical characteristics of these datasets are presented in Table~\ref{tab: dataset}. In our experiments, each dataset is split into training, validation, and test sets in a 6:2:2 ratio.
\begin{table}[h]
    \caption{Dataset Information}
    \label{tab: dataset}
    \centering
    \begin{tabular}{lrrrr}
        \toprule
        Dataset & Classes & Nodes & Edges & Features \\
        \midrule
        Cora~\cite{Sen_Namata_Bilgic_Getoor_Galligher_Eliassi-Rad_2017}        & 7  & 2708   & 10858   & 1433 \\
        Pubmed~\cite{Getoor_2012}      & 3  & 19717  & 88676   & 500  \\
        Citeseer~\cite{Getoor_2012}    & 6  & 3327   & 9464    & 3703 \\
        Facebook~\cite{Rozemberczki_Allen_Sarkar_2019}    & 4  & 22470  & 85501   & 4714 \\
        BlogCatalog~\cite{10.5555/2888116.2888372} & 6  & 5196   & 343486  & 8189    \\
        Flickr~\cite{10.1007/978-3-642-33765-9_59}      & 9  & 7575   & 479476  & 12047    \\
        Github~\cite{Rozemberczki_Allen_Sarkar_2019}      & 4  & 37700  & 144501  & 4005 \\
        \bottomrule
    \end{tabular}
\end{table}
\subsection{Baselines}
We compare the proposed MGHRL with two categories of baseline methods: graph neural networks (GNNs) and hypergraph neural networks (HNNs). The graph neural network methods include GCN \cite{kipf2017semi}, GAT \cite{velivckovic2018graph}, GraphSAGE \cite{hamilton2017inductive}, and GDGIN \cite{kong2022geodesic}. The hypergraph neural network methods include HyperGCN \cite{yadati2019hypergcn}, UniGCN \cite{huang2021unignn}, UniGAT \cite{huang2021unignn}, UniSAGE \cite{huang2021unignn}, HNHN \cite{dong2020hnhn}, HGNN \cite{feng2019hypergraph}, HGNN+ \cite{DBLP:journals/pami/GaoFJJ23}, CHNN \cite{CHNN}, ED-HNN \cite{DBLP:conf/iclr/WangYLWL23}, and HyperNAS \cite{DBLP:conf/aaai/LinP0J24}. Node classification accuracy is used as the evaluation metric to assess the performance of these methods.

\subsection{Experimental Setting}
\label{Experimental Setting}
Experiments are conducted on an Intel Xeon Gold 5218 CPU (2.30 GHz) with four Tesla V100 GPUs, using Python 3.10.9, PyTorch 1.13, and CUDA 11.6. 
The model is configured with four columns, a hidden dimension of 16, and eight attention heads per block in the multi-head attention mechanism. ReLU is used as the activation function. During training, the Adam optimizer minimizes the cross-entropy loss with a learning rate of 0.0001 and a regularization parameter of \( 5 \times 10^{-4} \).

\begin{table*}[th]
    \caption{Test accuracy (\%) on the benchmark datasets. Bolded results denote statistically significant improvements of our model over the baselines, with a p-value less than 0.01.}
    \label{tab:acc}
    \centering
    \begin{tabular}{lc@{\hspace{0.7em}}c@{\hspace{0.5em}}c@{\hspace{0.5em}}c@{\hspace{0.5em}}c@{\hspace{0.5em}}c@{\hspace{0.7em}}c@{\hspace{0.7em}}c}
        \toprule
        Models      & Cora   & PubMed  & Citeseer  & Facebook & BlogCatalog & Flickr & Github & AVG   \\
        \midrule
        GCN         & 82.32  & 84.26   & 69.66     & 92.68    & 79.05       & 56.96  & 84.57  & 78.50 \\
        GAT         & 84.26  & 84.66   & 70.97     & 92.09    & 80.21       & 60.46  & 85.04  & 79.67 \\
        GraphSAGE   & 84.08  & 84.29   & 70.72     & 92.61    & 80.82       & 69.30  & 84.07  & 80.84 \\
        GDGIN       & 81.21  & 82.37   & 71.62     & 91.36    & 75.86       & 73.53  & 85.04  & 80.14 \\
        HyperGCN    & 85.13  & 84.95   & 74.21     & 93.15    & 75.28       & 60.46  & 86.22  & 79.91 \\
        HNHN        & 85.95  & 85.01   & 72.81     & 93.24    & 74.23       & 62.64  & 85.68  & 79.93 \\
        UniGCN      & 85.09  & 84.11   & 74.51     & 93.18    & 87.11       & 61.65  & 85.66  & 81.62 \\
        UniGAT      & 85.26  & 84.74   & 74.12     & 93.54    & 87.53       & 77.88  & 86.10  & 84.17 \\
        UniSAGE     & 83.79  & 84.30   & 74.47     & 92.10    & 87.84       & 74.41  & 80.05  & 82.42 \\
        HGNN        & 85.43  & 85.12   & 72.42     & 93.56    & 88.23       & 79.10  & 86.10  & 84.28 \\
        HGNN+       & 85.63  & 85.32   & 74.72     & 93.68    & 88.17       & 79.14  & 86.58  & 84.75 \\
        CHNN        & 86.03  & 85.79   & 74.97     & 93.81    & 92.33       & 82.33  & 86.57  & 85.98 \\
        HyperNAS    & 83.90  & 81.30   & 74.10     & -        & -           & -      & -      & -     \\
        ED-HNN      & 86.58  & 86.23   & 75.32     & 93.88    & 88.71       & 83.87  & 86.68  & 85.90 \\
        \midrule
        \textbf{MGHRL(ours)} & \textbf{87.47} & \textbf{86.40} & \textbf{75.97} & \textbf{93.90} & \textbf{92.88} & \textbf{84.29} & \textbf{86.96} & \textbf{86.84} \\
        \bottomrule
    \end{tabular}
\end{table*}
\subsection{Comparison Experiments}
% Based on the results in Table \ref{tab:acc}, the following observations can be made: 
% Experimental results demonstrate that MGHRL consistently outperforms all baseline models across seven benchmark datasets. Notably, MGHRL achieves substantial improvements on the Facebook (93.90\%) and BlogCatalog (92.88\%) datasets. These results indicate that MGHRL is particularly effective at capturing high-order relationships, especially in complex social network datasets. Although HNNs (e.g., HGNN and HGNN+) generally outperform traditional models such as GCN and GAT, they still fall short of MGHRL. For example, HGNN+ achieves an average accuracy of 85.02\%, while MGHRL attains 86.84\%. This performance gap underscores the importance of multi-granularity hypergraph construction and network architecture in effectively learning from complex graph data. MGHRL captures more informative features at multiple granularities, leading to a notable improvement in performance.

% These results demonstrate that MGHRL consistently outperforms existing models across various graph datasets, especially excelling in learning complex, multi-level relationships within graphs.
Table \ref{tab:acc} compares the node classification accuracy of all evaluated models on seven benchmark datasets. The results reveal two key findings: (1) \textbf{Superior performance of MGHRL:} The proposed framework achieves the highest accuracy across all datasets, with significant improvements on social networks(up to 4.17\%). This improvement is attributed to MGHRL’s adaptive granular-ball hyperedge generation, which dynamically builds multi-granularity hypergraphs to model high-order relationships. (2) \textbf{Superiority of multi-granularity hypergraph:} While HNNs (e.g., HGNN, HGNN+) outperform standard GNNs, their reliance on predefined hyperedges restricts flexibility. In contrast, MGHRL dynamically refines graph structures via granular-ball splitting and integrates multi-scale features through a multi-column network, enhancing representation learning. These results demonstrate the efficacy of multi-granularity hypergraph modeling in real-world graph analysis.

\subsection{Ablation Studies}

To investigate the mechanisms of MGHRL, we conducted ablation studies on the Cora, PubMed, and Citeseer datasets using two variants: MGHRL - w/o \textsl{AGHG} replaces the AGHG component with a k-Hop strategy, while MGHRL - w/o \textsl{MGHN} utilizes HGNN in place of the MGHN structure. The results presented in Table \ref{tab: Ablation Studies} lead to the following observations:
\begin{itemize}
    \item \textbf{MGHRL - w/o \textsl{AGHG}} shows a significant performance drop across all datasets, highlighting the importance of the multi-granular framework in capturing intricate features and enhancing generalization.
    
    \item \textbf{MGHRL - w/o \textsl{MGHN}} results in a more substantial performance decline, particularly on the Citeseer dataset, indicating the crucial role of MGHN in capturing complex node-hyperedge relationships and distinguishing between diverse structural patterns.
\end{itemize}

\begin{table}[h]
  \caption{Results of the Ablation Study.}
  \label{tab: Ablation Studies}
  \centering
  \begin{tabular}{lrrr}
    \toprule
    Models & Cora & PubMed & Citeseer \\
    \midrule
    ours & 87.47 & 86.40 & 75.97   \\
    -w/o AGHG & 86.93 & 86.07 & 74.87 \\
    -w/o MGHN & 86.88 & 86.15 & 74.52  \\
    \bottomrule
  \end{tabular}
\end{table}
\begin{figure*}[t]
    \centering
    \begin{subfigure}[b]{0.3\textwidth}
        \centering
        \includegraphics[width=\linewidth]{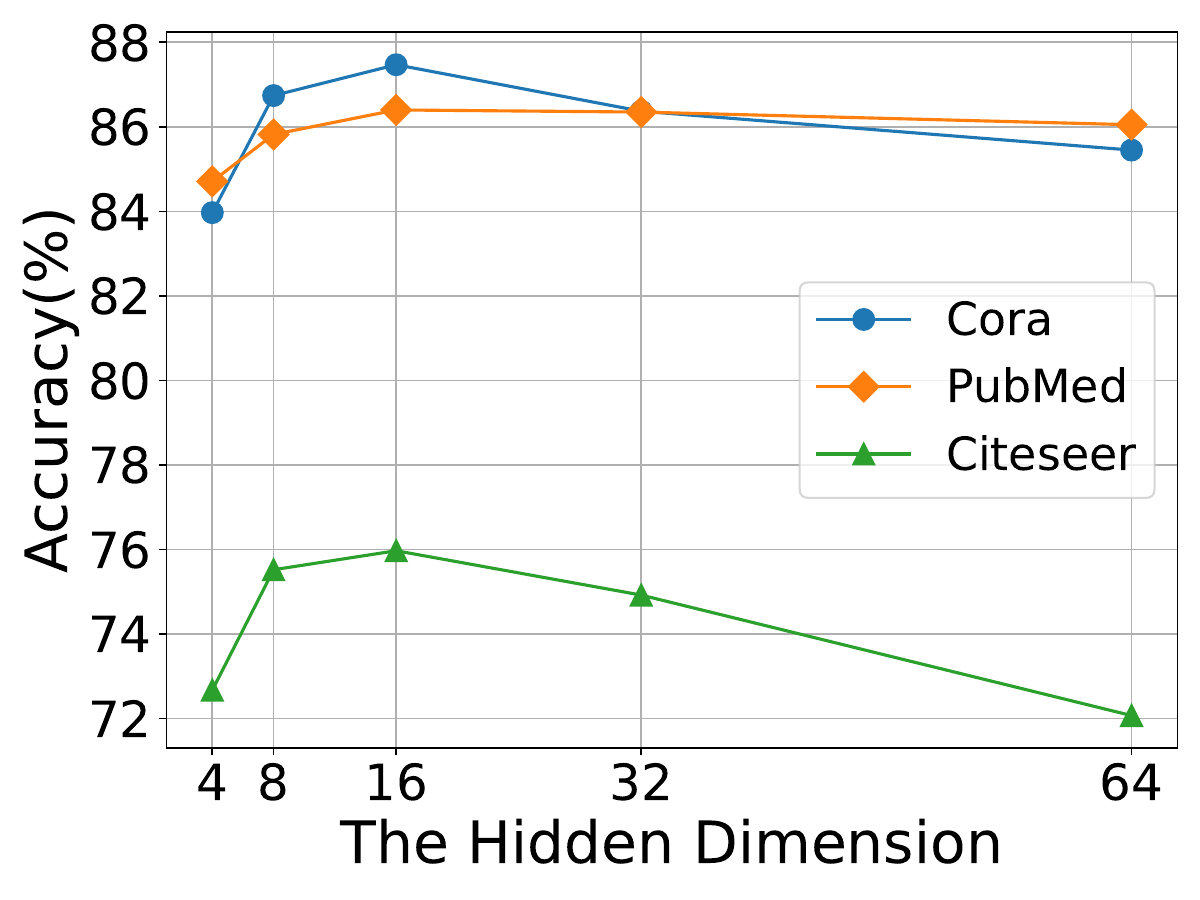}
        \caption{Hidden dimension impact.}
        \label{fig:hidden_dim}
    \end{subfigure}
    \hfill
    \begin{subfigure}[b]{0.3\textwidth}
        \centering
        \includegraphics[width=\linewidth]{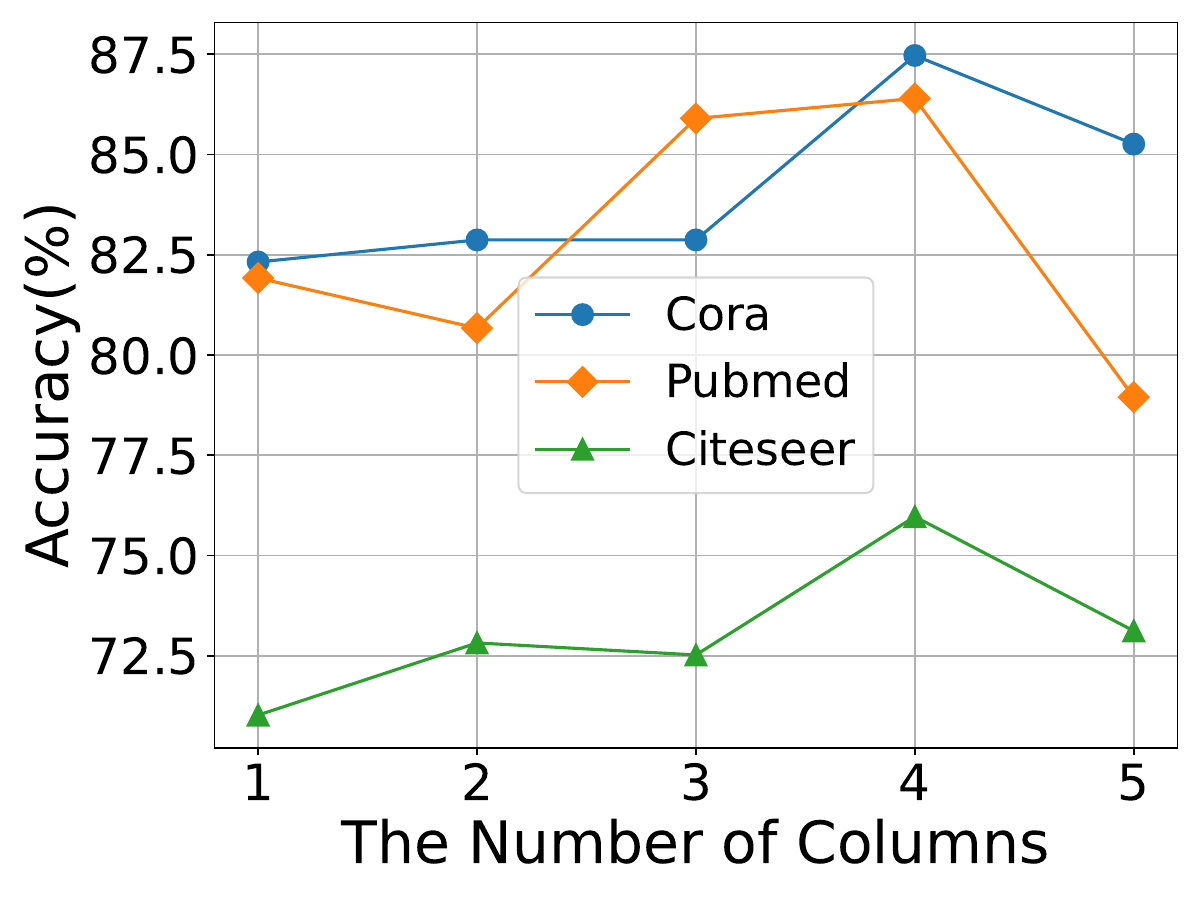}
        \caption{Column number impact.}
        \label{fig:num_columns}
    \end{subfigure}
    \hfill
    \begin{subfigure}[b]{0.3\textwidth}
        \centering
        \includegraphics[width=\linewidth]{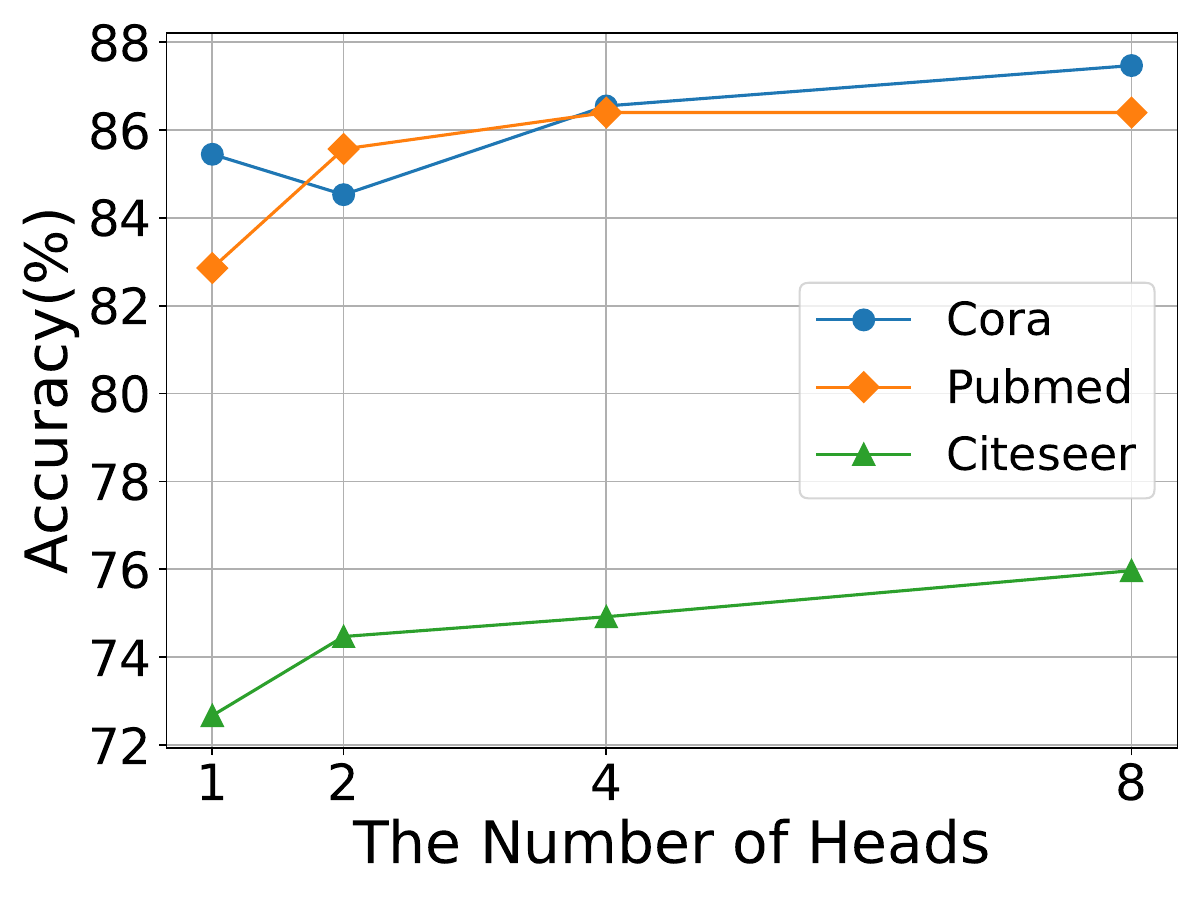}
        \caption{Head number impact.}
        \label{fig:num_heads}
    \end{subfigure}
    \caption{Impact of different parameters on model performance.}
    \label{fig:parameter_impact}
\end{figure*}

\begin{figure*}[t]
    \centering
    \begin{subfigure}[b]{0.3\textwidth}
        \centering
        \includegraphics[width=\linewidth]{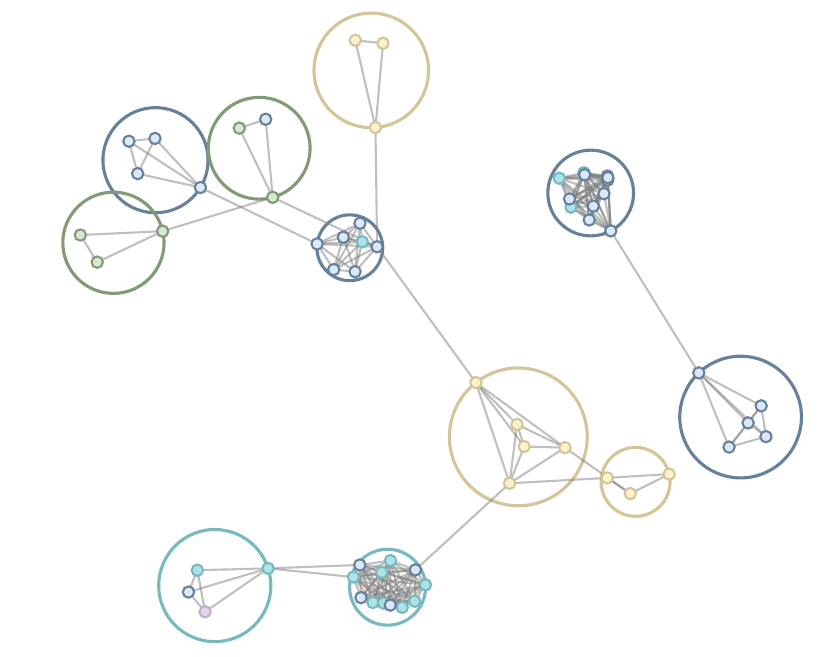}
        \caption{Cora}
        \label{fig:cora}
    \end{subfigure}
    \hfill
    \begin{subfigure}[b]{0.3\textwidth}
        \centering
        \includegraphics[width=\linewidth]{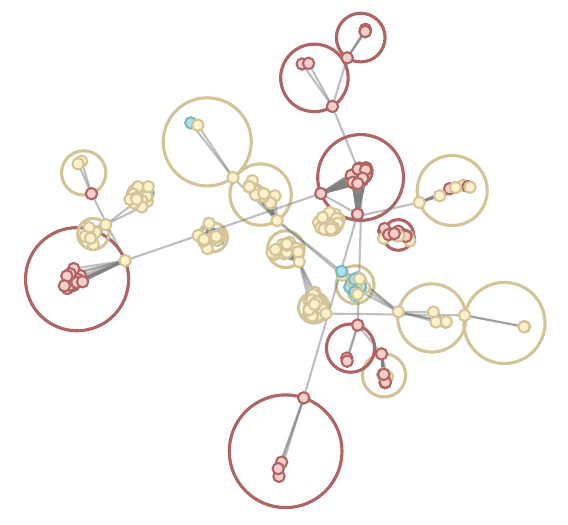}
        \caption{Pubmed}
        \label{fig:pubmed}
    \end{subfigure}
    \hfill
    \begin{subfigure}[b]{0.3\textwidth}
        \centering
        \includegraphics[width=\linewidth]{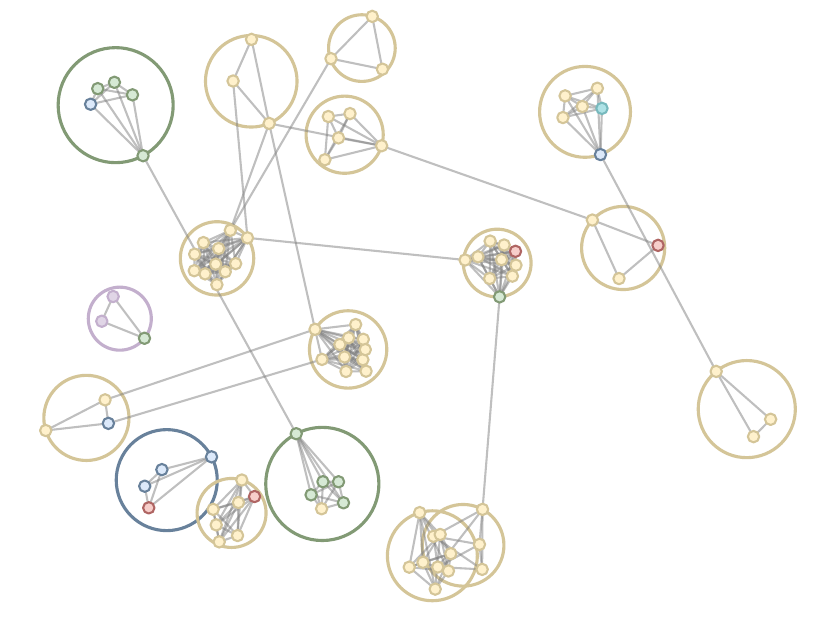}
        \caption{Citeseer}
        \label{fig:citeseer}
    \end{subfigure}
    \caption{Visualization of MGHRL performance across different datasets.}
    \label{fig:dataset_visualization}
\end{figure*}
\subsection{Hyper-parameter Sensitivity Analysis}

The adaptive identification of hyperedges with varying scales and types in the MGHRL framework enhances each hyperedge's unique characteristics, significantly improving the model's ability to capture complex relationships and structural patterns. This section examines the impact of three key factors—hidden dimension \(d\), number of columns \(c\), and number of attention heads \(h\)—on the performance of MGHRL, as illustrated in Figure~\ref{fig:parameter_impact}.

\textbf{Hidden Dimension \(d\):} Experiments show that as \(d\) increases from \{4, 8, 16, 32, 64\}, model performance improves, peaking at \(d = 16\), beyond which it slightly declines. This suggests that larger dimensions help capture patterns but may introduce noise or lead to overfitting.

\textbf{Number of Columns \(c\):} Performance improves with more columns, indicating that multi-column structures are better at capturing high-order relationships across granularities. However, when \(c = 5\), performance decreases, likely due to the over-smoothing of node features.

\textbf{Number of Attention Heads \(h\):} Performance consistently improves as the number of attention heads increases, highlighting the attention mechanism's ability to capture multi-granularity information, which enhances the learning of complex relationships.

\subsection{Case Study}
This section investigates the influence of hyperedges generated by the MGHRL method on three benchmark datasets (Cora, PubMed, and CiteSeer), focusing on their ability to model high-order relationships among nodes. As illustrated in Figure ~\ref{fig:dataset_visualization}, the visualization results demonstrate these effects through a representative subset of nodes, where distinct colors denote different node categories and circular markers highlight the hyperedges constructed by AGHG.
In these datasets, nodes represent papers and edges denote citations. Their complex structures involve high-order relationships beyond direct links, which traditional models often overlook.

AGHG generates multi-granularity hyperedges through adaptive granular-ball splitting, capturing both local and global topological features while revealing hidden high-order connections. In Cora, these hyperedges reveal intra-category clusters without direct citations, while in PubMed and CiteSeer, they expose cross-domain and interdisciplinary patterns.
This study demonstrates that the hyperedges generated by MGHRL effectively capture both local and global structures, providing deeper insights into complex relationships within citation networks.

\section{Conclusion}

In conclusion, the MGHRL method significantly enhances the modeling of complex relationships in graph-structured data. By leveraging adaptive granular-ball splitting and multi-column network structures, MGHRL effectively captures the intricate characteristics of hyperedges across multiple levels. This innovative approach enables a nuanced understanding of both local and global structures, resulting in superior performance in node classification tasks. MGHRL consistently outperforms state-of-the-art techniques across various benchmark datasets, demonstrating its robustness and effectiveness. Its exceptional adaptability positions MGHRL as a powerful and reliable tool for analyzing complex graph data, offering deeper insights into thematic and interdisciplinary connections that yield valuable, actionable outcomes across diverse domains.

\begin{ack}
Use unnumbered first level headings for the acknowledgments. All acknowledgments
go at the end of the paper before the list of references. Moreover, you are required to declare
funding (financial activities supporting the submitted work) and competing interests (related financial activities outside the submitted work).
More information about this disclosure can be found at: \url{https://neurips.cc/Conferences/2025/PaperInformation/FundingDisclosure}.

Do {\bf not} include this section in the anonymized submission, only in the final paper. You can use the \texttt{ack} environment provided in the style file to automatically hide this section in the anonymized submission.
\end{ack}

% \section*{References}
\bibliographystyle{unsrt}
\bibliography{reference}

% References follow the acknowledgments in the camera-ready paper. Use unnumbered first-level heading for
% the references. Any choice of citation style is acceptable as long as you are
% consistent. It is permissible to reduce the font size to \verb+small+ (9 point)
% when listing the references.
% Note that the Reference section does not count towards the page limit.
% \medskip

% {
% \small

% [1] Alexander, J.A.\ \& Mozer, M.C.\ (1995) Template-based algorithms for
% connectionist rule extraction. In G.\ Tesauro, D.S.\ Touretzky and T.K.\ Leen
% (eds.), {\it Advances in Neural Information Processing Systems 7},
% pp.\ 609--616. Cambridge, MA: MIT Press.

% [2] Bower, J.M.\ \& Beeman, D.\ (1995) {\it The Book of GENESIS: Exploring
%   Realistic Neural Models with the GEneral NEural SImulation System.}  New York:
% TELOS/Springer--Verlag.

% [3] Hasselmo, M.E., Schnell, E.\ \& Barkai, E.\ (1995) Dynamics of learning and
% recall at excitatory recurrent synapses and cholinergic modulation in rat
% hippocampal region CA3. {\it Journal of Neuroscience} {\bf 15}(7):5249-5262.
% }

%%%%%%%%%%%%%%%%%%%%%%%%%%%%%%%%%%%%%%%%%%%%%%%%%%%%%%%%%%%%

\clearpage
\appendix

\section{Proofs}
\label{Proofs}
\subsection{Proofs regarding Proposition \ref{prop.metricl}}
\label{app.proof.metric}
\begin{proof}
We verify the standard conditions required for $d(\cdot,\cdot)$ to define a valid metric on the graph-based domain $(\mathcal{V}, d)$:

\begin{itemize}
    \item \textbf{Non-negativity}: The shortest-path distance $\text{dist}(v,c)$ is always non-negative, i.e., $\text{dist}(v,c) \geq 0$. Therefore, $d(v,c) \geq 0$ holds.

    \item \textbf{Symmetry}: The shortest-path distance is symmetric with respect to node pairs: $\text{dist}(v,c) = \text{dist}(c,v)$.

    \item \textbf{Triangle inequality}: The shortest-path distance preserves the triangle inequality: for any $v, u, w \in \mathcal{V}$,
    \[
    d(v,w) \leq d(v,u) + d(u,w).
    \]
\end{itemize}

Hence, $d(\cdot,\cdot)$ satisfies the axioms of a metric and induces a valid metric topology $\tau_d$ over $\mathcal{V}$. The closed metric balls 
\[
\widetilde{\mathcal{G}}_i = \left\{ v \in \mathcal{V} \mid d(v, c_i) \leq r_{c_i} \right\}
\]
serve as basic neighborhoods in $\tau_d$. Consequently, the family of such balls forms a finite topological covering of $\mathcal{V}$, which enables localized, scalable, and topology-aware decomposition under the Granular-ball Computing framework.
\end{proof}

\subsection{Proofs regarding Proposition \ref{prop.granularball}}
\label{app.proof.granularball}
\begin{proof}
The proof is structured around verifying the three essential properties within the metric topology $(\mathcal{V}, \tau_d)$ induced by $d$:

\begin{itemize}
    \item \textbf{Coverage:} By construction, each Granular-ball $\widetilde{\mathcal{G}}_i$ is defined as a closed metric ball $\{v \in \mathcal{V} \mid d(v, c_i) \leq r_{c_i}\}$ centered at node $c_i \in \mathcal{V}$ with radius $r_{c_i} > 0$. The collection $\widetilde{\mathbb{G}}$ is selected to satisfy the topological covering condition:
    \begin{equation}
        \bigcup_{i=1}^n \widetilde{\mathcal{G}}_i \supseteq \mathcal{V},
    \end{equation}
    ensuring that every node $v \in \mathcal{V}$ is contained in at least one Granular-ball. This guarantees full domain coverage under the topology $\tau_d$.
    
    \item \textbf{Locality:} Since $d$ is a valid metric over $\mathcal{V}$, each node $v \in \mathcal{V}$ possesses a neighborhood basis consisting of open metric balls. Although $\widetilde{\mathcal{G}}_i$ is defined as a closed ball, it still contains all nodes within the radius $r_{c_i}$ from $c_i$, and in discrete spaces like graphs, closed balls serve as effective approximations to open neighborhoods. Consequently, each $\widetilde{\mathcal{G}}_i$ preserves local structural coherence and supports locality-aware computations on $\mathcal{G}$.

    \item \textbf{Basis Neighborhood:} In metric-induced topologies, closed balls form a base of closed neighborhoods. The finite collection $\widetilde{\mathbb{G}}$ with tunable centers and radii satisfies the neighborhood basis property: for any node $v \in \mathcal{V}$ and any neighborhood $U$ of $v$ in $\tau_d$, there exists at least one Granular-ball $\widetilde{\mathcal{G}}_i$ such that 
    \begin{equation}
        v \in \widetilde{\mathcal{G}}_i \subseteq U.
    \end{equation}
    This ensures that $\widetilde{\mathbb{G}}$ forms a legitimate basis of localized subspaces aligned with the topology.
\end{itemize}

Together, these three properties verify that the Granular-ball covering $\widetilde{\mathbb{G}}$ respects the topological structure induced by $d$, while enabling scalable, localized, and semantically meaningful decomposition of the graph domain. This forms the theoretical foundation for the Topological Granular-ball Decomposition framework and its application to graph representation and processing.

\end{proof}
\section{Time Complexity Analysis of AGHG Algorithm}
The total time complexity of AGHG is expressed as 
\[
\mathcal{O}(|\mathcal{V}|^{3/2} + |\mathcal{E}| \sqrt{|\mathcal{V}|}),
\]
where \( |\mathcal{V}| \) denotes the number of nodes and \( |\mathcal{E}| \) represents the number of edges in the input graph.

This complexity arises from the algorithm’s two-phase design. In the \textbf{coarse-grained initialization phase}, the algorithm identifies preliminary cluster centers through multiple breadth-first search (BFS) traversals. Specifically, it performs \( \sqrt{|\mathcal{V}|} \) iterations, with each iteration consisting of the selection of a high-degree node followed by a BFS expansion. Each BFS traversal explores the graph up to a predefined structural threshold (e.g., until a layer contains more than \( \sqrt{|\mathcal{V}|} \) nodes), resulting in a worst-case cost of 
\[
\mathcal{O}(|\mathcal{V}| + |\mathcal{E}|) 
\]
per traversal. Therefore, the cumulative time complexity of this phase is 
\[
\mathcal{O}(\sqrt{|\mathcal{V}|} \cdot (|\mathcal{V}| + |\mathcal{E}|)).
\]

In the \textbf{fine-grained refinement phase}, the algorithm recursively splits the coarse clusters (also referred to as granular-balls) using binary partitioning techniques. Each subcluster is refined through dual-source BFS traversals and local node reassignments, leading to an additional time complexity bounded by 
\[
\mathcal{O}(|\mathcal{V}| + |\mathcal{E}| \sqrt{|\mathcal{V}|}).
\] 
The use of a \( \sqrt{|\mathcal{V}|} \) scaling factor ensures that the number of recursive operations remains sublinear with respect to the total node count, while still maintaining meaningful partition granularity.

Consequently, the overall time complexity of AGHG remains subquadratic, making it particularly suitable for large-scale sparse and medium-density graphs. Its complexity bounds offer a favorable trade-off between computational efficiency and clustering accuracy when compared with classical graph clustering methods.

\section{Limitations and Future Research}
\label{Limitations}

Although the proposed \textbf{MGHRL} framework has demonstrated notable progress in modeling multi-granular high-order relationships within static hypergraphs, it still faces several limitations when applied to the increasingly common scenarios involving \textbf{temporal graphs}, such as human pose estimation~\cite{Xu_Zou_Lin_2022,LI2024106153,shang2024adamshyperadaptivemultiscalehypergraph} and traffic dynamics prediction~\cite{electronics13224435,8809901,10884295}.

\subsection{Limitations}

\paragraph{Limitations of Static Assumptions:}

The current MGHRL framework is primarily designed for static hypergraphs, where node relationships are assumed to remain invariant over time. Consequently, it does not explicitly model the temporal evolution of node and hyperedge structures. However, many real-world systems exhibit inherently dynamic characteristics, with relationships among entities evolving over time.

For example, in human pose estimation~\cite{Xu_Zou_Lin_2022,LI2024106153,shang2024adamshyperadaptivemultiscalehypergraph}, nodes can represent keypoints of the human body, and hyperedges may connect multiple joints simultaneously (e.g., arms and legs). In practice, these poses change continuously over time, and the temporal coupling between joints evolves accordingly. MGHRL, in its current form, lacks the capability to capture such temporal continuity and structural dynamics, thereby limiting its applicability to video-level dynamic human action recognition.

Similarly, in traffic prediction tasks~\cite{electronics13224435,8809901,10884295}, traffic flow within a road network exhibits significant temporal variability and strong spatial correlations across regions. The interdependence among different road segments can be naturally represented using hyperedges. However, these relationships are inherently dynamic and change over time. The hyperedge partitioning mechanism in MGHRL is statically defined, lacking the ability to perceive and adapt to the evolution of high-order regional structures, which diminishes its effectiveness in dynamic traffic modeling.

\paragraph{Scalability to Temporal Data:}

Moreover, the multi-granular modeling mechanism in MGHRL is not integrated with any temporal modeling component. In temporal graphs, nodes may transition from fine-grained local interactions to coarse-grained global collaborations as time progresses. However, the framework does not account for this dynamic evolution of granularity levels. For instance, in traffic systems, the spatial granularities may differ significantly between peak and off-peak hours. MGHRL lacks mechanisms to model such time-driven granularity transitions.

Additionally, MGHRL has not yet been empirically validated on widely used benchmark datasets for temporal graphs, such as NTU RGB+D for human activity recognition, or METR-LA and PEMS-BAY for urban traffic forecasting. Therefore, its generalization ability to real-world temporal tasks remains unclear and requires further investigation.

\subsection{Future Research Directions}

To address the above limitations and extend the applicability of MGHRL to dynamic systems, several future research directions are envisioned:

\paragraph{Learning Representations in Temporal Hypergraphs}

Future work can explore extending MGHRL for temporal hypergraph representation learning. One potential direction is to incorporate time-weighted graph distances or structural similarity measures that consider the temporal trajectories of node states. This would allow for the adaptive construction of temporally consistent hyperedges across successive time steps. In the context of human pose estimation, for example, the model should be capable of continuously tracking the time-varying coupling relationships between joints.

\paragraph{Integration with Temporal Graph Neural Networks}

MGHRL can also be integrated with existing temporal graph neural networks (TGNNs), such as TGAT and DySAT, to jointly model multi-granular relationships and temporal dynamics. Two promising directions include:
\begin{itemize}
  \item \textbf{Temporal Column Structures}: Incorporating recurrent (RNN) or temporal convolutional (TCN) columns within the MGHN architecture to enable effective feature aggregation across time steps.
  \item \textbf{Temporal Reversible Connections}: Enhancing the reversible connection mechanism in MGHRL to support the simultaneous propagation of granularity levels and temporal information, thereby improving temporal consistency and expressiveness.
\end{itemize}

\paragraph{Validation in Dynamic Systems}

Finally, it is essential to validate the improved framework on dynamic datasets, such as those representing social network evolution, traffic flow forecasting, and disease transmission modeling. This will help assess the framework's capability to:
\begin{itemize}
  \item Identify temporal high-order patterns, such as evolving community structures in social networks.
  \item Leverage historical multi-granular structures to predict future interactions.
\end{itemize}

\section{Broader Impacts}
\label{Broader Impacts}
This work aims to advance the understanding of hypergraph representation learning by introducing a novel multi-granularity modeling framework. We anticipate that the theoretical and empirical insights presented in this work will inspire future research in adaptive hypergraph construction and multi-scale graph neural architectures. This research does not involve human subjects, sensitive data, or ethically sensitive applications, and thus presents no known ethical concerns or negative societal impacts.

\end{document}